%% file: main.tex
\documentclass[conference,letterpaper]{IEEEtran}

\usepackage[utf8]{inputenc} 
\usepackage[T1]{fontenc}
\usepackage{url}
\usepackage{ifthen}
\usepackage[cmex10]{amsmath} 
\usepackage{cite}
\usepackage{bbm}
\usepackage{color}
\usepackage{comment}
\usepackage{amssymb}
\usepackage{amsmath}
\usepackage{amsthm}
\usepackage [english]{babel}
\usepackage [autostyle, english = american]{csquotes}
\usepackage{mathtools}
\usepackage[hang,flushmargin]{footmisc}

\DeclareMathOperator*{\argmin}{arg\,min}
\MakeOuterQuote{"}

\usepackage[short]{optidef}

\usepackage{stfloats}
\usepackage{multirow}
\usepackage{tikz}
\usepackage{pgfplots}
\usetikzlibrary{positioning,quotes,angles,patterns,intersections, pgfplots.fillbetween}

\makeatletter
\newtheorem*{rep@theorem}{\rep@title}
\newcommand{\newreptheorem}[2]{%
\newenvironment{rep#1}[1]{%
 \def\rep@title{#2 \ref{##1}}%
 \begin{rep@theorem}}%
 {\end{rep@theorem}}}
\makeatother
\newtheorem{theorem}{Theorem}
\newreptheorem{theorem}{Theorem}
\newtheorem{lemma}{Lemma}
\newreptheorem{lemma}{Lemma}
\newtheorem{prop}{Proposition}
\newreptheorem{prop}{Proposition}
\newtheorem{corollary}{Corollary}
\newreptheorem{corollary}{Corollary}

\newtheorem{definition}{Definition}

\usepackage{algorithm}
\usepackage[noend]{algpseudocode}
\usepackage{tikz}
\usepackage{pgfplots}
\usetikzlibrary{positioning,quotes,angles,patterns,intersections, pgfplots.fillbetween}

\begin{document}
\title{On the $\alpha$-loss Landscape in the Logistic Model}


\author{%
Tyler Sypherd${}^\ast$, Mario Diaz${}^\dagger$,  Lalitha Sankar${}^\ast$, and Gautam Dasarathy${}^\ast$ \\
${}^\ast$ Arizona State University, \texttt{\{tsypherd,lsankar,gautamd\}@asu.edu} \\
${}^\dagger$ Universidad Nacional Aut\'{o}noma de M\'{e}xico, \texttt{mario.diaz@sigma.iimas.unam.mx} \\
%
}
\maketitle

\begin{abstract}
\input{abstract}
\end{abstract}







\input{introduction}
\section{Preliminaries}
\input{AlphaLoss.tex}
\section{Landscape Evolution in the Logistic Model}

\input{Mainresults.tex}

\section{Concluding Remarks}
\input{Conclusions.tex}
\bibliographystyle{IEEEtran}
\bibliography{TS_ML}
\newpage

\end{document}

%% file: abstract.tex
We analyze the optimization landscape of a recently introduced tunable class of loss functions called $\alpha$-loss, $\alpha \in (0,\infty]$, in the logistic model.
This family encapsulates the exponential loss ($\alpha = 1/2$), the log-loss ($\alpha = 1$), and the 0-1 loss ($\alpha = \infty$) and contains compelling properties that enable the practitioner to discern among a host of operating conditions relevant to emerging learning methods. 
Specifically, we study the evolution of the optimization landscape of $\alpha$-loss with respect to $\alpha$ using tools drawn from the study of strictly-locally-quasi-convex functions in addition to geometric techniques. We interpret these results in terms of optimization complexity via normalized gradient descent.

%% file: introduction.tex
%
\section{Introduction}
The performance of a classification algorithm, in terms of accuracy, tractability, and convergence guarantees crucially depends on the choice of the loss function during training. Consider a feature vector $X \in \mathcal{X}$, an unknown finite-valued label $Y \in \mathcal{Y}$, and a hypothesis $h:\mathcal{X} \rightarrow \mathcal{Y} $. The canonical $0$-$1$ loss, given by $\mathbbm{1}[h(X) \neq Y]$, is considered an ideal loss function that captures the probability of incorrectly guessing the true label $Y$ using $h(X)$. However, since the $0$-$1$ loss is neither continuous nor differentiable, its applicability in state-of-the-art learning algorithms is highly restricted.

\textit{Surrogate} loss functions that approximate the $0$-$1$ loss such as log-loss, exponential loss, sigmoid loss, etc. have generated much interest \cite{bartlett2006convexity,masnadi2009design,lin2004note,nguyen2009,rosasco2004loss,nguyen2013algorithms,singh2010loss,tewari2007consistency,zhao2010convex,barron2019general,lin2017focal,mei2018landscape,sypherd2019tunable,sypherd2019alpha}. 
While early research was predominantly focused on convex losses \cite{bartlett2006convexity,rosasco2004loss,nguyen2009,lin2004note}, more recent works propose the use of non-convex losses as a means to moderate the behavior of an algorithm \cite{mei2018landscape,nguyen2013algorithms,masnadi2009design,barron2019general}. 
This is primarily due to the fact that modern learning models (e.g., deep learning) are inherently non-convex as they involve vast functional compositions \cite{goodfellow2016deep}; further, non-convex losses are also believed to provide increased robustness over convex losses \cite{mei2018landscape,barron2019general,masnadi2009design,nguyen2013algorithms}.
%

There have been numerous theoretical attempts to capture the non-convex optimization landscape
which is the loss surface induced by the learning model, underlying distribution, and the surrogate loss function itself \cite{mei2018landscape,hazan2015beyond,li2018visualizing,nguyen2017loss,fu2018guaranteed,liang2018understanding,engstrom2019exploring,chaudhari2018deep}.
Notably, Hazan \textit{et al.} \cite{hazan2015beyond} propose the notion of \textit{Strict-Local-Quasi-Convexity} (SLQC) to parametrically quantify quasi-convex functions, and provide convergence guarantees for the efficiency of the Normalized Gradient Descent (NGD) algorithm (originally introduced in \cite{nesterov1984minimization}) optimizing such functions.


In \cite{sypherd2019tunable}, Sypherd \textit{et al.} introduce a tunable class of loss functions called $\alpha$-loss, $\alpha\in [1,\infty]$, which includes log-loss ($\alpha = 1$) and the soft 0-1 loss ($\alpha=\infty$); they prove that it satisfies many desirable properties for surrogate losses including the notion of classification-calibration \cite{bartlett2006convexity}. In the extended version of this paper \cite{sypherd2019alpha}, Sypherd \textit{et al.} extend $\alpha$-loss to the range $\alpha\in (0,\infty]$ which includes exponential loss ($\alpha = 1/2$); they prove that the extended range of $\alpha$ also induces classification-calibrated losses and has desirable convexity characteristics.
Further, they show experimentally that, relative to log-loss ($\alpha=1$), $\alpha>1$ achieves increased robustness to noise while $\alpha<1$ achieves better accuracy for imbalanced classes. 

In this paper, we present three main contributions for the logistic model (the hypothesis class of sigmoid soft classifiers \cite{friedman2001elements}): (i) we show that the expected risk under $\alpha$-loss is strongly convex for $\alpha\in(0,1]$ (under mild distribution assumptions); (ii) we provide, in a quantitative manner, bounds for the evolution of the SLQC parameters of the expected risk of $\alpha$-loss as $\alpha$ increases, which is most useful in a neighborhood of $\alpha_{0} = 1$ when combined with the first result; (iii) we study a saturation effect of $\alpha$-loss in the logistic model, \textit{i.e.}, how the distance between the expected risk for $\alpha\geq1$ quickly resembles the expected risk of $\alpha=\infty$.
As a byproduct of the analysis in the second point, we prove an equivalent form of the SLQC definition that can be of independent interest.
Based on our theoretical analysis into the evolution of the optimization landscape with respect to $\alpha$, we ultimately posit that there
is a small range of $\alpha$ useful to the practitioner, thereby drastically reducing the search for the optimal value of $\alpha$ in practice. Further, via the saturation effect, we argue that this narrow search in $\alpha$ is sufficient for the logistic model. 

%% file: AlphaLoss.tex
\subsection[binary classification]{$\alpha$-loss Definition and Interpretations} \label{section:marginprops}
%
\begin{definition}
Let $\mathcal{P}(\mathcal{Y})$ be the set of probability distributions over $\mathcal{Y}$. For $\alpha \in (0,\infty]$, we define $\alpha$-loss for $\alpha \in (0,1) \cup (1,\infty)$, $l^{\alpha}:\mathcal{Y} \times \mathcal{P}(\mathcal{Y}) \rightarrow \mathbb{R}_{+}$ as
\begin{equation} \label{eq:def1}
l^{\alpha}(y,P_{Y}) := \frac{\alpha}{\alpha - 1}\left[1 - P_{Y}(y)^{1 - 1/\alpha}\right],
\end{equation} 
and, by continuous extension, $l^{1}(y,P_{Y}) := -\log{P_{Y}(y)}$ and $l^{\infty}(y,P_{Y}) := 1 - P_{Y}(y)$.
\end{definition}
%
Note that $l^{1/2}(y,P_{Y}) := P_{Y}^{-1}(y) - 1$. We refer to $l^{1/2}$ as the soft exponential loss and $l^{\infty}$ as the soft 0-1 loss; observe that $l^{1}$ recovers log-loss. 
For $(y,P_{Y})$ fixed, note that $l^{\alpha}(y,P_{Y})$ is continuous in $\alpha$. 
The above definition of $\alpha$-loss presents a tunable class of loss functions that value the probabilistic estimate of the label differently as a function of $\alpha$.

Consider random variables $(X,Y) \sim P_{X,Y}$. Observing $X$, one can construct an estimate $\hat{Y}$ of $Y$ such that $Y - X - \hat{Y}$ form a Markov chain. One can use expected $\alpha$-loss $\mathbb{E}_{X,Y}[l^{\alpha}(Y,P_{\hat{Y}|X})]$, hence called $\alpha$-risk, to quantify the effectiveness of the estimated posterior $P_{\hat{Y}|X}$. In particular, 
\begin{equation} \label{inf1}
\mathbb{E}_{X,Y}\left[l^{1}(Y,P_{\hat{Y}|X})\right] = \mathbb{E}_{X}\left[H(P_{Y|X=x},P_{\hat{Y}|X=x})\right],
\end{equation}
where $H(P,Q) := H(P) + D_{\textnormal{KL}}(P\|Q)$ is the cross-entropy between $P$ and $Q$. Similarly,
\begin{equation} \label{inf2}
\mathbb{E}_{X,Y}[l^{\infty}(Y,P_{\hat{Y}|X})] = \mathbb{P}[Y \neq \hat{Y}],
\end{equation}
i.e., the expected $\alpha$-loss for $\alpha = \infty$ equals the probability of error. 
Recall that the expectation of the 0-1 loss is also the probability of error \cite{shalev2014understanding}; thus, we say that $\alpha$-loss for $\alpha = \infty$ is a soft version of the 0-1 loss.
%
The following result by Liao \textit{et al.} provides an explicit characterization of the risk-minimizing posterior under $\alpha$-loss.
\begin{prop}[{\!\cite[Lemma 1]{liao2018tunable}}] \label{Prop:Liao}
For each $\alpha\in [1,\infty]$, the minimal $\alpha$-risk is
\begin{equation}
\min_{P_{\hat{Y}|X}} \mathbb{E}_{X,Y}\big[ l^{\alpha}(Y,P_{\hat{Y}|X})\big] = \frac{\alpha}{\alpha -1}\left(1 - e^{\frac{1-\alpha}{\alpha}H_{\alpha}^{A}(Y|X)}\right),
 \end{equation}
 where $H_{\alpha}^{A}(Y|X) = \dfrac{\alpha}{1 - \alpha} \log{\sum\limits_{y}\Big(\sum\limits_{x} P_{X,Y}(x,y)^{\alpha}\Big)^{1/\alpha}}$ is the Arimoto conditional entropy of order $\alpha$ \cite{arimoto1977information}. The resulting unique minimizer, $P^{*}_{\hat{Y}|X}(y|x)$, is the $\alpha$-tilted true posterior
 \begin{equation} \label{eq:tilteddistribution}
 P^{*}_{\hat{Y}|X}(y|x) = \dfrac{P_{Y|X}(y|x)^{\alpha}}{\sum\limits_{y} P_{Y|X} (y|x)^{\alpha}}.
\end{equation}
\end{prop}
%
%
The proof of Proposition~\ref{Prop:Liao} can be found in \cite{liao2018tunable} and is easily extended to the case where $\alpha \in (0,1)$.
For $\alpha=\infty$, minimizing the corresponding risk leads to making a single guess on the most likely label; on the other hand, for $\alpha=1$, such a risk minimization involves minimizing the average log-loss, and therefore, obtaining the true posterior belief. 

We note that $\alpha$-loss exhibits different operating conditions through the choice of $\alpha$; see \cite{sypherd2019alpha} for experimental consideration of robustness and class imbalance trade-offs. With respect to \eqref{eq:tilteddistribution}, as $\alpha$ increases from 1 to $\infty$, $\alpha$-loss increasingly limits the effect of low probability outcomes; on the other hand, as $\alpha$ decreases from 1 towards 0, $\alpha$-loss places increasingly higher weights on low probability outcomes. 


\subsection{Strict-Local-Quasi-Convexity}
We briefly review \textit{Strict-Local-Quasi-Convexity} which was introduced by Hazan \textit{et al.} in \cite{hazan2015beyond}. 
For $\theta_0\in\mathbb{R}^d$ and $r>0$, we let $\mathbb{B}_{d}(\theta_0,r) := \{\theta\in\mathbb{R}^d : \|\theta-\theta_0\| \leq r\}$. 
For simplicity, we let $\mathbb{B}_{d}(r)= \mathbb{B}_{d}(\mathbf{0},r)$; also note that all norms are Euclidean.

\begin{definition} \label{def:SLQC}
Let $\theta, \theta_{0} \in \mathbb{R}^{d}$, $\kappa, \epsilon > 0$. We say that $f: \mathbb{R}^{d} \rightarrow \mathbb{R}$ is $(\epsilon, \kappa, \theta_{0})$-Strictly-Locally-Quasi-Convex (SLQC) in $\theta$, if at least one of the following applies:
\begin{enumerate}
\item $f(\theta) - f(\theta_{0}) \leq \epsilon$.
\item $\|\nabla f(\theta)\| > 0$, and for every $\theta' \in \mathbb{B}_{d}(\theta_{0}, \epsilon/\kappa)$ it holds that $\langle \nabla f(\theta), \theta' - \theta \rangle \leq 0$.
\end{enumerate}
\end{definition}

Intuitively, if $\theta_0$ is fixed, 
then, for every $\theta$, either $f(\theta)$ is $\epsilon$-close to $f(\theta_{0})$ or the constraint cone induced by the set of $\theta'$ about $\theta_{0}$ requires quasi-convex functional descent behavior.
This relaxed notion of quasi-convexity aligns with a natural adaptation of the Gradient Descent (GD) algorithm, namely, Normalized Gradient Descent (NGD) \cite{hazan2015beyond} as summarized in Algorithm \ref{algo:NGD} below.
\begin{algorithm} 
\caption{Normalized Gradient Descent (NGD)}\label{algo:NGD}
\begin{algorithmic}[1]
\State \textbf{Input:} $T$ Iterations, $\theta_{1} \in \mathbb{R}^{d}$, learning rate $\eta > 0$
\For {$t = 1, 2, \ldots, T$}
\State Update: $\theta_{t+1} = \theta_{t} - \eta \dfrac{\nabla f(\theta_{t})}{\|\nabla f(\theta_{t})\|}$
\EndFor
\State \textbf{Return} $\bar{\theta}_{T} = \argmin\limits_{\theta_{1}, \ldots, \theta_{T}} f(\theta_{t})$
\end{algorithmic}
\end{algorithm}
Similar to the convergence guarantees for GD for convex functions, the following result by Hazan \textit{et al.} summarizes such guarantees of NGD for SLQC functions.
\begin{prop}[{\!\cite[Theorem 4.1]{hazan2015beyond}}] \label{prop:NGDiterations}
Fix $\epsilon > 0$, let $f: \mathbb{R}^{d} \rightarrow \mathbb{R}$, and $\theta^{*} = \argmin_{\theta \in \mathbb{R}^{d}} f(\theta)$. If $f$ is $(\epsilon, \kappa, \theta^{*})$-SLQC in every $\theta \in \mathbb{R}^{d}$, then by running Algorithm \ref{algo:NGD} with $\eta = \epsilon/\kappa$ and $T \geq \kappa^{2}\|\theta_{1} - \theta^{*}\|^{2}/\epsilon^{2}$, we have  $f(\bar{\theta}_{T}) - f(\theta^{*}) \leq \epsilon$.
\end{prop}
%
For an $(\epsilon,\kappa,\theta_0)$-SLQC function, a smaller $\epsilon$ provides better optimality guarantees. Given $\epsilon>0$, smaller $\kappa$ leads to faster optimization as the number of required iterations increases with $\kappa^2$. Finally, by using projections, NGD can be easily adapted to work over convex and closed sets including  $\mathbb{B}_d(r)$. 

%% file: Mainresults.tex
\begin{figure}[h] 
    \centering
    \centerline{\includegraphics[width=1\linewidth]{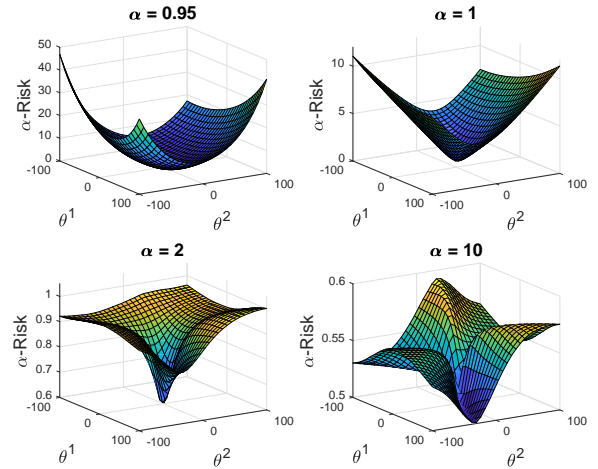}}
    \caption{The landscape of $\alpha$-loss ($R_{\alpha}$ for $\alpha = 0.95, 1, 2, 10$) in the logistic model, where features are normalized and $r=100$, for a 2D-GMM with $\mathbb{P}[Y=-1] = 0.12 = 1 - \mathbb{P}[Y=1]$, $\mu_{X|Y=-1} =[-0.18, 1.49]$, $\mu_{X|Y=1} = [-0.01,0.16]$, $\Sigma_{-1} = [3.20,   -2.02; -2.01,    2.71]$, and $\Sigma_{1} = [4.19,    1.27;  1.27,   0.90]$.}
    \label{fig:fourlandscapes}
\end{figure}
%
%
In this section, 
we quantify the optimization complexity of NGD by characterizing the SLQC constants ($\epsilon$ and $\kappa$) of the $\alpha$-risk within a neighborhood of $\alpha=1$ (log-loss) in the logistic model. 
For $\alpha \leq 1$, we find that the $\alpha$-risk is strongly convex under mild distributional assumptions; for $\alpha>1$ we reinterpret the SLQC definition to provide bounds on SLQC constants as $\alpha$ increases. Finally, we provide a result which characterizes a saturation effect of the $\alpha$-risk in the logistic model, \textit{i.e.}, the fact that the $\alpha$-risk observes uniform convergence with respect to $\alpha=\infty$ as $\alpha$ increases.
%

Prior to stating our main results, we clarify the setting and provide necessary definitions. 
%
%
%
Let $X \in \mathbb{B}_{d}(1) := \{x\in\mathbb{R}^d:\|x\| \leq 1\}$ be the normalized feature, $Y \in \{-1,+1\}$ the label and $S_{n} = \{(X_{i},Y_{i}) : i = 1,\ldots,n\}$ the training dataset where, for each $i\in\{1,\ldots,n\}$, the samples $(X_{i},Y_{i})$ are independently drawn according to an unknown distribution $P_{X,Y}$. 
For a given $r>0$, we consider the logistic model and its associated hypothesis class $\mathcal{G} = \{g_\theta:\theta\in\mathbb{B}_{d}(r)\}$, composed of parameterized soft classifiers $g_{\theta}$ such that
%
\begin{equation}
g_{\theta}(x) = \sigma(\langle \theta,x \rangle),
\end{equation}
with $\sigma: \mathbb{R} \rightarrow [0,1]$ being the sigmoid function given by
\begin{equation}
\label{eq:DefSigmoid}
    \sigma(z) = \frac{1}{1+e^{-z}}.
\end{equation}
For convenience, we present the following short form of $\alpha$-loss in the logistic model which is equivalent to the expanded expression in \cite{sypherd2019tunable}.
For $\alpha \in (0,\infty]$, $\alpha$-loss is given by
\begin{equation} \label{eq:alphadefLR}
l^{\alpha}(y,g_{\theta}(x)) = \frac{\alpha}{\alpha-1} \left[1 - g_{\theta}(yx)^{1-1/\alpha} \right].
\end{equation}
For $\alpha = 1$, $l^{1}$ is the logistic loss and we recover logistic regression by optimizing this loss. Further, note that in this setting $y \cdot \langle x, \theta \rangle$ is the margin, and \eqref{eq:alphadefLR} is convex for $\alpha \in (0,1]$ and quasi-convex for $\alpha > 1$ in $y \cdot \langle x, \theta \rangle$; see the extended version \cite{sypherd2019alpha} for proofs of these facts.

For $\theta \in \mathbb{B}_{d}(r)$, we define the $\alpha$-risk $R_\alpha$ as the risk of \eqref{eq:alphadefLR},  
\begin{equation} \label{eq:alphariskdefLR}
R_{\alpha}({\theta}) := \mathbb{E}_{X,Y}[l^{\alpha}(Y,g_{\theta}(X))].
\end{equation}
The $\alpha$-risk \eqref{eq:alphariskdefLR} is plotted for several values of $\alpha$ in a two-dimensional Gaussian Mixture Model (GMM) in Figure~\ref{fig:fourlandscapes}.
Further, observe that, for all $\theta\in\mathbb{B}_d(r)$,
\begin{equation}
R_\infty(\theta) := \mathbb{E}_{X,Y}[l^{\infty}(Y,g_{\theta}(X))] = \mathbb{P}[Y \neq \hat{Y}_\theta],
\end{equation}
where $\hat{Y}_\theta$ is a random variable such that for all $x\in\mathbb{B}_d(1)$, $\mathbb{P}[\hat{Y}_\theta = 1|X=x] = g_\theta(x)$. 

In order to study the landscape of the $\alpha$-risk, we compute the gradient and Hessian of \eqref{eq:alphadefLR}, by employing the following useful properties of the sigmoid 
\begin{align} \label{eq:sigprop1}
\dfrac{d}{dz} \sigma(z) = \sigma(z)(1-\sigma(z)); && \sigma(-z) = 1-\sigma(z).
\end{align}
%
Indeed, a straightforward computation shows that
\begin{equation} \label{eq:alphaderLR}
\frac{\partial}{\partial \theta^{j}} l^{\alpha}(y,g_{\theta}(x)) = \left[- y g_\theta(yx)^{1-1/\alpha}(1-g_\theta(yx))\right]x^{j}, 
\end{equation} 
where $\theta^{j}, x^{j}$ denote the $j$-th components of $\theta$ and $x$, respectively. 
Thus, the gradient of $\alpha$-loss in \eqref{eq:alphadefLR} is
\begin{equation} \label{eq:alphagradLR}
\nabla_\theta l^{\alpha}(Y,g_\theta(X)) = F_{1}(\alpha,\theta,X,Y)X,
\end{equation}
where $F_{1}(\alpha,\theta,x,y)$ is the expression within brackets in \eqref{eq:alphaderLR}.
Another straightforward computation yields
\begin{equation} \label{eq:alphader2LR}
\nabla^{2}_{\theta}l^{\alpha}(Y,g_{\theta}(X)) = F_{2}(\alpha,\theta,X,Y)XX^{T},
\end{equation}
where $F_{2}$ is given by
\begin{align} \label{eq:F2}
\nonumber 
F_2(\alpha,\theta,x,y) &= g_{\theta}^{1 - \alpha^{-1}}(yx) \left(g_{\theta}'(yx) - \left(1-\alpha^{-1}\right) g_{\theta}^{2}(-yx) \right).
\end{align}
%

We now turn our attention to the case where $\alpha \in (0,1]$; we find that for this regime, $R_{\alpha}$ is strongly convex; see Figure~\ref{fig:fourlandscapes}. 
Prior to stating the result, for two matrices $A,B \in \mathbb{R}^{d\times d}$, we let $\geq$ denote the Loewner (partial) order in the positive semi-definite cone. That is, we write $A \geq B$ when $A-B$ is a positive semi-definite matrix. 
For a matrix $A \in \mathbb{R}^{d\times d}$, let $\lambda_1(A),\ldots,\lambda_d(A)$ be its eigenvalues.
Finally, we recall that a function is $m$-strongly convex if and only if its Hessian has minimum eigenvalue $m \geq 0$ \cite{boyd2004convex}. 
\begin{theorem} \label{Thm:SLQClessthan1}
Let $\Sigma := \mathbb{E}[XX^{T}]$. If $\alpha \in (0,1]$, then $R_{\alpha}(\theta)$ 
is $\Lambda(\alpha,r) \min_{i \in [d]}\lambda_{i}\left(\Sigma\right)$-strongly convex in $\theta \in \mathbb{B}_{d}(r)$,
where 
\begin{align}
\Lambda(\alpha,r) = \sigma^{1-1/\alpha}(r)\left(\sigma'(r) - \left(1 - \alpha^{-1} \right)\sigma^{2}(-r)\right).
\end{align}
\end{theorem}
\begin{proof}
%
\noindent For each $\alpha \in (0,1]$, it can readily be shown that each component of $F_{2}(\alpha,\theta,x,y)$ is positive and monotonic in $\langle \theta,x \rangle$, which implies that $F_{2}(\alpha,\theta,x,y) \geq \Lambda(\alpha,r) \geq 0$.
%
%
Now, consider $R_{\alpha}(\theta) = \mathbb{E}[l^{\alpha}(Y,g_{\theta}(X))]$. We have
\begin{align}
\nonumber \nabla_{\theta}^{2} R_{\alpha}(\theta) &= \mathbb{E}_{X,Y}[\nabla_{\theta}^{2} l^{\alpha}(Y,g_{\theta}(X))] \\
\nonumber &= \mathbb{E}_{X,Y}[F_{2}(\alpha,\theta,X,Y) XX^{T}] \\ 
\label{eq:thm1_1} &\geq \Lambda(\alpha,r) \mathbb{E}[XX^{T}] \\
\label{eq:thm1_2} &= \Lambda(\alpha,r) \Sigma \geq 0,
\end{align}
where we used an identity of positive semi-definite matrices for \eqref{eq:thm1_1} (see, e.g., \cite[Ch.~7]{horn2012matrix}); for \eqref{eq:thm1_2}, we used the fact that $\Lambda(\alpha,r) \geq 0$ and we recognize that $\Sigma$ is positive semi-definite as it is the 
autocorrelation of the random vector $X \in \mathbb{B}_{d}(1)$ (see, e.g., \cite[Ch.~7]{papoulis2002probability}). We also note that $\min_{i \in [d]}\lambda_{i}\left(\Sigma\right) \geq 0$ (see, e.g., \cite[Ch.~7]{horn2012matrix}).
Thus, $\nabla_{\theta}^{2} R_{\alpha}(\theta)$ is positive semi-definite for every $\theta \in \mathbb{B}_{d}(r)$. 
Therefore, since $\lambda_{\min}(\nabla^{2} R_{\alpha}(\theta)) \geq \Lambda(\alpha,r) \min_{i \in [d]}\lambda_{i}\left(\Sigma\right) \geq 0$ for every $\theta \in \mathbb{B}_{d}(r)$ \cite[Corollary~4.3.12]{horn2012matrix}, we have that $R_{\alpha}$ is $\Lambda(\alpha,r) \min_{i \in [d]}\lambda_{i}\left(\Sigma\right)$-strongly convex for $\alpha \in (0,1]$.
%
%
\end{proof}
Observe that for $r > 0$, $\Lambda(\alpha,r)$ is monotonically decreasing in $\alpha$. 
Therefore, $R_{\alpha}$ becomes more strongly convex as $\alpha$ approaches zero.
%
%
It can be shown that $R_{\alpha}$ 
is $C_{r,\alpha}$-Lipschitz in $\theta$ where $C_{r,\alpha} := \sigma(r)(1-\sigma(r))^{1-1/\alpha}$. Thus, in conjunction with Theorem~\ref{Thm:SLQClessthan1} and a result by Hazan \textit{et al.} in \cite{hazan2015beyond} (after Definition 3) which holds by assuming $\Sigma>0$, we provide the following corollary which explicitly characterizes SLQC constants of $R_{\alpha}$ for $\alpha \in (0,1]$.
\begin{corollary} \label{Cor:SLQClessthan1}
If $0 < \alpha \leq 1$, $\Sigma > 0$, and $\theta_0 \in \mathbb{B}_{d}(r)$, then, for every $\epsilon>0$, the $\alpha$-risk $R_{\alpha}$ is $(\epsilon, C_{r,\alpha},\theta_0)$-SLQC in $\theta\in\mathbb{B}_{d}(r)$ where $C_{r,\alpha} = \sigma(r)(1-\sigma(r))^{1-1/\alpha}$.
\end{corollary}
%
As $\alpha$ tends to zero, $C_{r,\alpha}$ tends to infinity which implies that the learning rate of NGD, $\eta_{\alpha} = \epsilon/\kappa_{\alpha} = \epsilon/C_{r,\alpha}$ also tends to zero. Thus, by Proposition~\ref{prop:NGDiterations}, the number of iterations of NGD, $T_{\alpha}$, tends to infinity as $\alpha$ tends to zero. Therefore, for $\alpha \in (0,1]$, there is a trade-off in the desired strong-convexity of $R_{\alpha}$ and the computational complexity of NGD.



Next, we study the evolution of SLQC parameters of $R_{\alpha}$ in a neighborhood of $\alpha=1$ as we increase $\alpha$.
Since $R_{\alpha}$ tends more towards the probability of error (expectation of $0$-$1$ loss) as $\alpha$ approaches infinity, we find that SLQC constants  deteriorate and the computational complexity of NGD increases as we increase $\alpha$.
Our next main result leverages the following novel lemma, which is a structural result for general differentiable functions that provides an alternative formulation of the second requirement of SLQC functions in Definition~\ref{def:SLQC}; proof details and illustrations can be found in the extended version \cite{sypherd2019alpha}. 
\begin{lemma} \label{anglelemma}
Assume that $f:\mathbb{R}^d\to\mathbb{R}$ is differentiable, $\theta_{0} \in \mathbb{R}^d$ and $\rho>0$. If $\theta\in\mathbb{R}^d$ is such that $\|\theta-\theta_0\|>\rho$, then the following are equivalent:
\begin{itemize}
    \item[1.] $\langle -\nabla f(\theta), \theta' - \theta \rangle \geq 0$ for all $\theta' \in \mathbb{B}_{d}\left(\theta_{0},\rho\right)$,
    \item[2.] $\langle -\nabla f(\theta), \theta_{0} - \theta \rangle \geq \rho \|\nabla f(\theta)\|$.
\end{itemize}
\end{lemma}
Intuitively, Lemma~\ref{anglelemma} reformulates the SLQC requirement that the gradient points in the `right' direction into an expression which is reminiscent of a Cauchy-Schwarz inequality. 

We now present two Lipschitz inequalities which will be useful in the sequel.
In the extended version \cite{sypherd2019alpha}, Sypherd \textit{et al.} show that for $\alpha \in [1,\infty]$, $R_{\alpha}$ is $L_{r}$-Lipschitz in $\alpha$ where
\begin{equation}
\label{eq:DefLr}
\quad L_{r} := \dfrac{(r+\log{2})^{2}}{2}.
\end{equation}
It can similarly be shown that for $\alpha \in [1,\infty]$, $\nabla R_{\alpha}$ is $J_{r}$-Lipschitz in $\alpha^{-1} \in [0,1]$ where 
\begin{equation}
\label{eq:DefJr}
     J_{r} := (r+\log{2})\sigma(r).
\end{equation}
%
%
Finally, for ease of notation, let 
\begin{equation}
I_{\alpha_{0},\epsilon_{0},r}(\theta_{0}) = \inf_{\theta \in \mathbb{B}_{d}(r) \atop R_{\alpha_{0}}(\theta) - R_{\alpha_{0}}(\theta_{0}) > \epsilon_{0}} \|\nabla R_{\alpha_{0}}(\theta)\|.
\end{equation}
Using Lemma~\ref{anglelemma} and the Lipschitz relations 
\eqref{eq:DefLr} and \eqref{eq:DefJr}, we provide the following result which gives precise bounds on the degradation of SLQC constants for any initial $\alpha_{0} \in [1,\infty]$. 
\begin{theorem}\label{thm:SLQCresult1}
Let $\alpha_{0} \in [1,\infty]$, $\epsilon_{0}, \kappa_{0}>0$, and $\theta_{0} \in \mathbb{B}_{d}(r)$.  
%
If $R_{{\alpha_{0}}}$ is $(\epsilon_{0},\kappa_{0},\theta_{0})$-SLQC in $\theta \in \mathbb{B}_{d}(r)$, 
and
%
%
%
%
\begin{equation} \label{eq:radchange}
0 \leq \alpha-\alpha_{0} < \dfrac{\alpha_{0}^{2} I_{\alpha_{0},\epsilon_{0},r}(\theta_{0})}{2J_{r}\left(1 + r \frac{\kappa_{0}}{\epsilon_{0}}\right)}, 
\end{equation}
then $R_{{\alpha}}$ is $(\epsilon,\kappa,\theta_{0})$-SLQC in $\theta \in \mathbb{B}_{d}(r)$ with 
\begin{equation}
\epsilon = \epsilon_{0} + 2 L_{r}(\alpha-\alpha_{0}),
\end{equation}
and 
\begin{align} \label{eq:thmrho}
\dfrac{\epsilon}{\kappa} = \frac{\epsilon_{0}}{\kappa_{0}} \left(1  - \frac{\left(1 +2r\frac{\kappa_{0}}{\epsilon_{0}}\right)J_r(\alpha-\alpha_{0})}{\alpha\alpha_{0}I_{\alpha_{0},\epsilon_{0},r}(\theta_{0}) - J_{r}(\alpha-\alpha_{0})}\right).
\end{align}
\end{theorem}
\begin{proof}
For ease of notation let $\rho_{0} = \frac{\epsilon_{0}}{\kappa_{0}}$ and $\rho = \frac{\epsilon}{\kappa}$. Let $\theta \in \mathbb{B}_{d}(r)$ be arbitrary and consider the following cases.

\noindent \textbf{Case 1}: If $R_{\alpha_{0}}(\theta) - R_{\alpha_{0}}(\theta_{0}) \leq \epsilon_{0}$, then,
\begin{align}
\begin{split} 
\nonumber R_{\alpha}(\theta) - R_{\alpha}(\theta_{0}) ={}& R_{\alpha}(\theta) - R_{\alpha_{0}}(\theta) + R_{\alpha_{0}}(\theta) \\
\nonumber &- R_{\alpha_{0}}(\theta_{0}) + R_{\alpha_{0}}(\theta_{0}) - R_{\alpha}(\theta_{0}) 
\end{split} \\
\leq{}& L_{r}(\alpha-\alpha_{0}) + \epsilon_{0} + L_{r}(\alpha-\alpha_{0}).
\end{align}
Since $\epsilon_{0} +2L_{r}(\alpha-\alpha_{0}) = \epsilon$, we have $R_{\alpha}(\theta) - R_{\alpha}(\theta_{0}) \leq \epsilon$. 

\noindent \textbf{Case 2}:
If $R_{\alpha_{0}}(\theta) - R_{\alpha_{0}}(\theta_{0}) > \epsilon_{0}$, then, since $R_{\alpha_{0}}$~is $(\epsilon_{0},\kappa_{0},\theta_{0})$-SLQC in $\theta$ by assumption, we have that $\|\nabla R_{\alpha_{0}}(\theta)\| > 0$, and for every $\theta' \in \mathbb{B}(\theta_{0}, \rho_{0})$ it holds that $\langle \nabla R_{\alpha_{0}}(\theta), \theta_{0} - \theta \rangle \leq 0$.
By Lemma~\ref{anglelemma}, we want to prove that 
\begin{equation}
\langle-\nabla R_{\alpha}(\theta), \theta_{0} - \theta \rangle \geq \rho \|\nabla R_{\alpha}(\theta) \|,
\end{equation}
for $\rho$ given by \eqref{eq:thmrho}.
By the Cauchy-Schwarz inequality, 
\begin{align}
\begin{split}
\nonumber \langle -\nabla R_{\alpha}(\theta),\theta_{0} - \theta \rangle \geq{}& \langle -\nabla R_{\alpha_{0}}(\theta), \theta_{0} - \theta \rangle \\
\nonumber &- \|\nabla R_{\alpha}(\theta) - \nabla R_{\alpha_{0}}(\theta)\|\|\theta_{0} - \theta\|  
\end{split}\\
\geq{} \rho_{0} \|\nabla & R_{\alpha_{0}}(\theta)\| - J_{r}(\alpha_{0}^{-1} - \alpha^{-1}) 2r,
\end{align}
since $\nabla R_{\alpha}$ is $J_{r}$-Lipschitz in $\alpha^{-1}$ and $\theta_{0}-\theta \in \mathbb{B}_{d}(2r)$, and since $R_{\alpha_{0}}$ is SLQC, we apply Lemma~\ref{anglelemma}.
For ease of notation, we temporarily let $\Delta = J_{r}(\alpha_{0}^{-1} - \alpha^{-1})$.
Continuing, we have
\begin{align}
\begin{split}
\nonumber \rho_{0} \|\nabla R_{\alpha_{0}}(\theta)\| - \Delta 2r \geq{}& \rho_{0} \|\nabla R_{\alpha}(\theta)\| - \Delta 2r \\
\nonumber &- \rho_{0} \|\nabla R_{\alpha_{0}}(\theta) - \nabla R_{\alpha}(\theta)\| 
\end{split} \\
\label{eq:thm2_2} \geq{}& \rho_{0} \|\nabla R_{\alpha}(\theta)\| - \Delta (\rho_{0} + 2r),
\end{align}
which follows by the reverse triangle inequality and since $\nabla R_{\alpha}(\theta)$ is $J_{r}$-Lipschitz in $\alpha^{-1}$.
Further, we have that
\begin{align} \label{eq:thm2_3}
0 < I_{\alpha_{0},\epsilon_{0},r}(\theta_{0}) - J_{r}(\alpha_{0}^{-1} - \alpha^{-1}) \leq  \|\nabla R_{\alpha}(\theta)\|,
\end{align}
which follows by the reverse triangle inequality, by the fact that $\nabla R_{\alpha}(\theta)$ is $J_{r}$-Lipschitz in $\alpha^{-1}$, and the definition of $\alpha$ in \eqref{eq:radchange} since $\alpha < \alpha_{0}^{2}I_{\alpha_{0},\epsilon_{0},r}(\theta_{0})J_{r}^{-1} + \alpha_{0}$ and $\alpha_{0}^{2} \leq \alpha \alpha_{0}$.
Thus, returning to \eqref{eq:thm2_2}, we let $\Gamma = \Delta (\rho_{0} + 2r)$ and $I = I_{\alpha_{0},\epsilon_{0},r}(\theta_{0})$ for ease of notation and we have 
\begin{align}
\nonumber \rho_{0} \|\nabla &R_{\alpha}(\theta)\| - \Gamma = \|\nabla R_{\alpha}(\theta)\| \left(\rho_{0} - \frac{\Gamma}{\|\nabla R_{\alpha}(\theta)\|}\right) \\ 
\label{eq:thm2_4}&\geq  \|\nabla R_{\alpha}(\theta)\| \left(\rho_{0} - \frac{(\rho_{0}+2r)J_r}{I(\alpha_{0}^{-1} - \alpha^{-1})^{-1} - J_{r}}\right)
\end{align}
where we used the inequality in \eqref{eq:thm2_3}.
Since we assume that 
\begin{align}
0 \leq \alpha-\alpha_{0} < \dfrac{\alpha_{0}^{2} I}{2J_{r}\left(1 + r \rho_{0}^{-1}\right)},
\end{align}
returning to \eqref{eq:thm2_4}, it can be shown using $\alpha_{0}^2 \leq \alpha \alpha_{0}$ that 
\begin{align}
\frac{(1+2r\rho_{0}^{-1})J_r}{I(\alpha_{0}^{-1} - \alpha^{-1})^{-1} - J_{r}} < 1.
\end{align}
Therefore, we finally obtain that 
\begin{align}
\langle-\nabla R_{\alpha}(\theta), \theta_{0} - \theta \rangle \geq \rho \|\nabla R_{\alpha}(\theta) \|,
\end{align}
where $\rho > 0$ is given by 
\begin{equation}
\rho = 
\rho_{0} \left(1 - \frac{(1+2r\rho_{0}^{-1})J_r}{I(\alpha_{0}^{-1} - \alpha^{-1})^{-1} - J_{r}}\right)
\end{equation}
~as desired.
\end{proof}
%
Combining Corollary~\ref{Cor:SLQClessthan1} and Theorem~\ref{thm:SLQCresult1}, we provide the following corollary which quantifies the evolution of SLQC constants for $\alpha_{0}=1$ as $\alpha$ increases.
\begin{corollary}\label{cor:SLQCresult=1}
Let $\Sigma>0$, $\alpha_{0}=1$, $\epsilon_{0}>0$, and $\theta_{0} \in \mathbb{B}_{d}(r)$.  
If 
\begin{equation}
0 \leq \alpha-1 < \dfrac{I_{\alpha_{0},\epsilon_{0},r}(\theta_{0})}{2J_{r}\left(1 + r \frac{\sigma(r)}{\epsilon_{0}}\right)},
\end{equation}
then $R_{{\alpha}}$ is $(\epsilon,\kappa,\theta_{0})$-SLQC in $\theta \in \mathbb{B}_{d}(r)$ with 
\begin{equation}
\epsilon = \epsilon_{0} + 2 L_{r}(\alpha-1),
\end{equation}
and 
\begin{equation} \label{eq:corrho}
\dfrac{\epsilon}{\kappa} = \frac{\epsilon_{0}}{\sigma(r)} \left(1  - \frac{\left(1 +2r\frac{\sigma(r)}{\epsilon_{0}}\right)J_r(\alpha-1)}{\alpha I_{\alpha_{0},\epsilon_{0},r}(\theta_{0}) - J_{r}(\alpha-1)}\right).
\end{equation}
\end{corollary}
\begin{figure}[h]
    \centering
    \centerline{\includegraphics[width=1\linewidth]{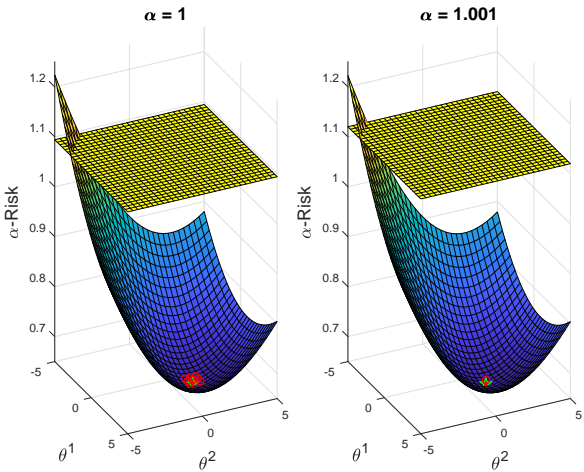}}
    \caption{The landscape of $\alpha$-loss, $R_{\alpha}$ ($\alpha = 1, 1.001$) in the logistic model, where the features are normalized and $r = 5$, for a 2D-GMM with $\mathbb{P}[Y=1]=\mathbb{P}[Y=-1]$, $\mu_{X|Y=-1} = [0.4,0.4]$, $\mu_{X|Y=1} = [1,1]$, $\Sigma = [3, 0.2; 0.2, 1.5]$. For $\alpha=1$, the red region depicts $\epsilon_{0}/\kappa_{0}$ which is calculated using Corollary~\ref{Cor:SLQClessthan1} about $\theta_{0}$, where $\theta_{0}$ is set to be the global minimum of $R_{1}$ and is depicted by the star; for illustrative purposes, we set $\epsilon_{0} = 0.4$ and it is depicted by the yellow plane. For $\alpha=1.001$, the red region depicts $\epsilon/\kappa$ about $\theta_{0}$ (the star) and $\epsilon$ is also depicted by the yellow plane; both quantities approximate the bounds given by Corollary~\ref{cor:SLQCresult=1}.} \label{fig:2DLandscape_ball}
\end{figure}
An illustration of the degradation of SLQC constants as specified by Corollary~\ref{cor:SLQCresult=1} for increasing $\alpha$ is presented in Figure~\ref{fig:2DLandscape_ball} for a two-dimensional GMM.
Intuitively, we find that for a fixed $\theta_{0} \in \mathbb{B}_{d}(r)$, increasing $\alpha$ is equivalent to reducing the radius of the $\epsilon/\kappa$ ball about $\theta_{0}$ and increasing the value of $\epsilon$. Both of these effects hinder the optimization process and increase the required number of iterations of NGD as stated in Proposition~\ref{prop:NGDiterations}. 


While the learning practitioner would ultimately like to approximate the intractable 0-1 loss (approximated by $\alpha = \infty$), 
the bounds presented in Theorem~\ref{thm:SLQCresult1} suggest that the computational complexity of NGD quickly worsens as $\alpha$ increases.
Fortunately, in the logistic model, $\alpha$-loss exhibits a saturation effect whereby smaller values of $\alpha$ resemble the landscape induced by $\alpha = \infty$.
More concretely, the saturation effect of $\alpha$-loss is the fact that the uniform distance between $R_{\alpha}$ and $R_{\infty}$ decreases geometrically in $\alpha$ as summarized by the following lemma.

\begin{lemma} \label{lemma:saturationLR}
If $\alpha,\alpha'\in[1,\infty]$, then for all $\theta \in \mathbb{B}_{d}(r)$,
\begin{equation}
|R_{\alpha}(\theta) - R_{\alpha'}(\theta)| \leq L_{r} \left|\frac{1}{\alpha} - \frac{1}{\alpha'}\right|,
\end{equation}
where $L_{r}$ is given in \eqref{eq:DefLr}.
\end{lemma}
See Figure~\ref{fig:saturation} for an illustration which depicts how quickly the landscape for $\alpha > 1$ resembles the $\alpha = \infty$ landscape.
\begin{figure}[h]
    \centering
    \centerline{\includegraphics[width=1\linewidth]{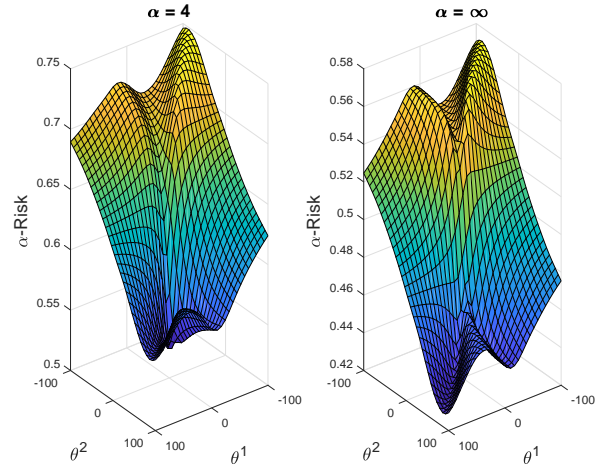}}
   \caption{An illustration of the saturation effect of $\alpha$-loss ($R_{\alpha}$ for $\alpha = 4, \infty$) in the logistic model, where features are normalized and $r = 100$, for a 2D-GMM 
   with $\mathbb{P}[Y=-1] = 0.61 = 1 - \mathbb{P}[Y=1]$, $\mu_{X|Y=-1} =[-0.14, 0.21]$, $\mu_{X|Y=1} = [0.06,0.43]$, $\Sigma_{-1} = [0.38, 0.25; 0.25, 3.17]$, and $\Sigma_{1} = [2.07, -1.62;  -1.62, 1.97]$.}
   %
   %
    \label{fig:saturation}
\end{figure}



%% file: Conclusions.tex
%
In this work, we analyze the evolution of the $\alpha$-loss landscape in the logistic model by examining different regimes of $\alpha$. 
As $\alpha$ approaches zero, $R_{\alpha}$ becomes more strongly convex (see Theorem~\ref{Thm:SLQClessthan1}), but the computational complexity of NGD increases since the Lipschitz constant of $R_{\alpha}$ grows. 
As $\alpha$ approaches infinity, $R_{\alpha}$ becomes more non-convex since SLQC parameters degrade (see Theorem~\ref{thm:SLQCresult1}), which also increases the computational complexity of NGD; however, accuracy increases since the landscape of $\alpha$-loss tends towards that of the $\infty$-loss, \textit{i.e.}, the 0-1 loss.
Combining Corollary~\ref{Cor:SLQClessthan1} and Theorem~\ref{thm:SLQCresult1}, we provide explicit bounds to quantify the evolution of SLQC parameters of $R_{\alpha}$ in the logistic model for $\alpha$ in a neighborhood of $1$ (see Corollary \ref{cor:SLQCresult=1}). 
Using a moderately large $\alpha$, $\alpha$-loss leads to similar performance as the desired, and computationally harder to optimize, 0-1 loss (see Lemma~\ref{lemma:saturationLR}); this is a saturation effect of $\alpha$-loss in the logistic model.
Therefore, for the logistic model, we ultimately posit that there is a narrow range of $\alpha$ useful to the practitioner.